\documentclass[letterpaper,english]{elsarticle}
\usepackage[T1]{fontenc}
\usepackage[latin9]{inputenc}
\usepackage{float}
\usepackage{amsmath}
\usepackage{amssymb}
\usepackage{graphicx}

\makeatletter

\pdfpageheight\paperheight
\pdfpagewidth\paperwidth

\usepackage{url}            

\usepackage{amsmath, amsthm}
\usepackage{algorithm, algpseudocode}

\renewcommand\[{\begin{equation}}
\renewcommand\]{\end{equation}} 

\newtheorem{theorem}{Theorem}[section]

\newtheorem{lemma}[theorem]{Lemma}

\newtheorem{assumption}[theorem]{Assumption}

\newcommand{\interior}[1]{%
  {\kern0pt#1}^{\mathrm{o}}%
}

\makeatother

\usepackage{babel}
\begin{document}

\title{Probabilistic bounds on neuron death in deep rectifier networks}

\author{Blaine Rister$_{a}$, Daniel L. Rubin$_{b}$}

\address{$_{a}$Stanford University, Department of Electrical Engineering.
email: blaine@stanford.edu. Corresponding author.\\
$_{b}$Stanford University, Department of Radiology (Biomedical Informatics
Research). email: dlrubin@stanford.edu.\\
1265 Welch Rd\\
Palo Alto, CA, USA 94305}
\begin{abstract}
Neuron death is a complex phenomenon with implications for model trainability:
the deeper the network, the lower the probability of finding a valid
initialization. In this work, we derive both upper and lower bounds
on the probability that a ReLU network is initialized to a trainable
point, as a function of model hyperparameters. We show that it is
possible to increase the depth of a network indefinitely, so long
as the width increases as well. Furthermore, our bounds are asymptotically
tight under reasonable assumptions: first, the upper bound coincides
with the true probability for a single-layer network with the largest
possible input set. Second, the true probability converges to our
lower bound as the input set shrinks to a single point, or as the
network complexity grows under an assumption about the output variance.
We confirm these results by numerical simulation, showing rapid convergence
to the lower bound with increasing network depth. Then, motivated
by the theory, we propose a practical sign flipping scheme which guarantees
that the ratio of living data points in a $k$-layer network is at
least $2^{-k}$. Finally, we show how these issues are mitigated by
network design features currently seen in practice, such as batch
normalization, residual connections, dense networks and skip connections.
This suggests that neuron death may provide insight into the efficacy
of various model architectures.
\end{abstract}
\maketitle{}

\begin{keywords}ReLU networks, neuron death, hyperparameters, probability,
statistics, machine learning

\end{keywords}

\pagebreak{}

\section{Introduction}

Despite the explosion of interest in deep learning over the last decade,
network design remains largely experimental. Engineers are faced with
a wide variety of design choices, including preprocessing of the input
data, the number of layers in the network, the use of residual and
skip connections between the layers, the width or number of neurons
in each layer, and which types of layers are used. Even with modern
parallel processing hardware, training a deep neural network can take
days or even weeks depending on the size and nature of the input data.
The problem is particularly acute in applications like biomedical
image analysis \citep{Insee:2021:nnU-net}. This imposes practical
limits on how much experimentation can be done, even with the use
of automatic model searching tools \citep{Hutter:2018:autoMLBook}.
In these cases, theoretical analysis can shed light on which models
are likely to work\textit{ a priori}, reducing the search space and
accelerating innovation.

This paper explores the model design space through the lens of neuron
death. We focus on the rectified linear unit (ReLU), which is the
basic building block of the majority of current deep learning models
\citep{Glorot:2011:ReluNetworks,Goodfellow:2016:DLBook,Krizhevsky:2012:ImageNetCW}.
The ReLU neuron has the interesting property that it maps some data
points to a constant function, at which point they do not contribute
to the training gradient. The property that certain neurons focus
on certain inputs may be key to the success of ReLU neurons over the
sigmoid type. However, a ReLU neuron sometimes maps all data to a
constant function, in which case we say that it is dead. If all the
neurons in a given layer die, then the whole network is rendered untrainable.
In this work, we show that neuron death can help explain why certain
model architectures work better than others. While it is difficult
to compute the probability of neuron death directly, we derive upper
and lower bounds by a symmetry argument, showing that these bounds
rule out certain model architectures as intractable.

Neuron death is a well-known phenomenon which has inspired a wide
variety of research directions. It is related to the unique geometry
of ReLU networks, which have piecewise affine or multi-convex properties
depending on whether the inference or training loss function is considered
\citep{Balestriero:2019:powerDiagram,Montufar:2014:linearRegionsNN,Rister:2017:piecewiseConvexity}.
Many activation functions have been designed to preclude neuron death,
yet the humble ReLU continues to be preferred in practice, suggesting
these effects are not completely understood \citep{Cheng:2020:PDeLU,Djork:2016:ELU,Goodfellow:2013:MaxOut,He:2015:HeInitialization,Ramachandran:2018:SwishNonlinearity}.
Despite widespread knowledge of the phenomenon, there are surprisingly
few theoretical works on the topic. Several works analyzed the fraction
of living neurons in various training scenarios. Wu et al.~studied
the empirical effects of various reinforcement leaning schemes on
neuron death, deriving some theoretical bounds for the single-layer
case \citep{wu2018:AdaptiveNetworkScaling}. Ankevist et al.~studied
the effect of different training algorithms on neuron death, both
empirically as well as theoretically using a differential equation
as a model for network behavior \citep{arnekvist:2020:dyingReluMomentum}.
To our knowledge, a pair of recent works from Lu and Shin et al.~were
the first to discuss the problem of an entire network dying at initialization,
and the first to analyze neuron death from a purely theoretical perspective
\citep{lu2019:dyingRelu,shin2019:dyingRelu2}. They argued that, as
a ReLU network architecture grows deeper, the probability that it
is initialized dead goes to one. If a network is initialized dead,
the partial derivatives of its output are all zero, and thus it is
not amenable to differential training methods such as stochastic gradient
descent \citep{LeCun:1998:EfficientBackprop}. This means that, for
a bounded width, very deep ReLU networks are nearly impossible to
train. Both works propose upper bounds on the probability of living
initialization, suggesting new initialization schemes to improve this
probability. However, there is much left to be said on this topic,
as bounds are complex and the derivations include some non-trivial
assumptions. 

In this work, we derive simple upper and lower bounds on the probability
that a random ReLU network is alive. Our proofs are based on a rigorous
symmetry argument requiring no special assumptions. Our upper bound
rigorously establishes the result of Lu et al., while our lower bound
establishes a new positive result, that a network can grow infinitely
deep so long as it grows wider as well \citep{lu2019:dyingRelu}.
We show that the true probability agrees with our bounds in the extreme
cases of a single-layer network with the largest possible input set,
or a deep network with the smallest possible input set. We also show
that the true probability converges to our lower bound along any path
through hyperparameter space such that neither the width nor depth
is bounded. Our proof of the latter claim requires an assumption about
the output variance of the network. We also note that our lower bound
is exactly the probability that a single data point is alive. All
of these results are confirmed by numerical simulations.

Finally, we discuss how information loss by neuron death furnishes
a compelling interpretation of various network architectures, such
as residual layers (ResNets), batch normalization, and skip connections
\citep{He:2016:ResNet,Ioffe:2015:batchNorm,Shelhamer:2017:FCN}. This
analysis provides \textit{a priori} means of evaluating various model
architectures, and could inform future designs of very deep networks,
as well as their biological plausibility. Based on this information,
we propose a simple sign-flipping initialization scheme guaranteeing
with probability one that the ratio of living training data points
is at least $2^{-k}$, where $k$ is the number of layers in the network.
Our scheme preserves the marginal distribution of each parameter,
while modifying the joint distribution based on the training data.
We confirm our results with numerical simulations, suggesting the
actual improvement far exceeds the theoretical minimum. We also compare
this scheme to batch normalization, which offers similar, but not
identical, guarantees. 

The contributions of this paper are summarized as follows:
\begin{enumerate}
\item New upper and lower bounds on the probability of ReLU network death,
\item proofs of the optimality of these bounds,
\item interpretation of various neural network architectures in light of
the former, and
\item a tractable initialization scheme preventing neuron death.
\end{enumerate}

\section{Preliminary definitions}

Given an input dimension $d$, with weights $a\in\mathbb{R}^{d}$
and $b\in\mathbb{R}$, and input data $x\in\mathbb{R}^{n}$, a ReLU
\textbf{neuron} is defined as 
\[
f(x)=\max\left\{ a\cdot x+b,0\right\} .
\]
A ReLU \textbf{layer} of with $n$ is just the vector concatenation
of $n$ neurons, which can be written $g(x)=\max\{Ax+b,0\}$ with
parameters $A\in\mathbb{R}^{n\times n}$ and $b\in\mathbb{R}^{n}$,
where the maximum is taken element-wise. A ReLU \textbf{network} with
$k$ layers is $F_{k}(x)=g_{k}\circ\dots\circ g_{1}(x)$, the composition
of the $k$ layers. The parameters of a network are denoted $\theta=(A_{1},b_{1},\dots,A_{n},b_{n})$,
a point in $\mathbb{R}^{N(n,k)}$ where $N(n,k)=k(n^{2}+n)$ is the
total number of parameters in the network. To simplify the proofs
and notation, we assume that the width of each layer is the same throughout
a network, always denoted $n$.

In practice, neural networks parameters are often initialized from
some random probability distribution at the start of training. This
distribution is important to our results. In this work, as in practice,
we always assume that $\theta$ follows a \textbf{symmetric}, zero-mean
probability density function (PDF). That is, the density of a parameter
vector $\theta$ is not altered by flipping the sign of any component.
Furthermore, all components of $\theta$ are assumed to be statistically
independent and identically distributed (IID), except where explicitly
stated otherwise. We sometimes follow the practical convention that
$b=0$ for all layers, which sharpens the upper bound slightly, but
most of our results hold either way.

We sometimes refer to the response of a network layer before the rectifying
nonlinearity. The pre-ReLU response of the $k^{th}$ layer is denoted
$\tilde{F}_{k}(x)$, and consists of the first $k-1$ layers composed
with the affine part of the $k^{th}$ layer, without the final maximum.
For short-hand, the response of the $l^{th}$ neuron in the $k^{th}$
layer is denoted $F_{k.l}$, while the pre-ReLU response is $\tilde{F}_{k.l}$.

Let $S_{0}\subseteq\mathbb{R}^{n}$ be the input domain of the network.
Let $S_{k}=f_{k}\circ\dots\circ f_{1}(S_{0})$, the image of $S_{0}$
under the $k^{th}$ layer, which gives the domain of $F_{k+1}$. Let
the random variable $x\in S_{0}$ denote the input datum, the precise
distribution of which is not relevant for this paper. We say that
$x$ is \textbf{dead }at layer $k$ if the Jacobian $\mathcal{D}F_{k}(x)$
is the zero matrix,\footnote{For convenience, assume the convention that a partial derivative is
zero whenever $\tilde{F}_{k,l}(x)=0$, in which case it would normally
be undefined.} taken with respect to the parameters $(A_{k},b_{k})$. By the chain
rule, $x$ is then dead in any layer after $k$ as well. Note that
if $b=0$, this is equivalent to the statement that $F_{k,l}$ is
the zero vector, while for $b\ne0$ it can be any constant.

The \textbf{dead set} of a ReLU neuron is
\begin{equation}
D_{f}=\{x:a\cdot x+b\le0\}\label{eq:dead_set_neuron}
\end{equation}
which is a half-space in $\mathbb{R}^{n}$. The dead set of a layer
is just the intersection of the half-space of each neuron, a convex
polytope, which is possibly empty or even unbounded. For the sake
of simplicity, this paper follows the convention that the number of
features $n$ is identical in each layer. In this case, the dead set
is actually an affine cone, as there is exactly one vertex which is
the intersection of $n$ linear equations in $\mathbb{R}^{n}$, except
for degenerate cases with probability zero. Furthermore, if we follow
the usual practice that $b=0$ at initialization, then the dead set
is a convex cone with vertex at the origin, as seen in figure \ref{fig:dead_sets}. 

For convenience, we denote the dead set of the $l^{th}$ neuron in
the $k^{th}$ layer as $D_{k,l}$, while the dead set of the whole
$k^{th}$ layer is $D_{k}=\cap_{l=1}^{n}D_{k,l}$. If $\tilde{F}_{k}(x)\in D_{k}$,
and $k$ is the least layer in which this occurs, we say that $x$
is \textbf{killed} by that layer. A dead neuron is one for which all
of $S_{0}$ is dead. In turn, a dead layer is one in which all neurons
are dead. Finally, a dead network is one for which the final layer
is dead, which is guaranteed if any intermediate layer dies.

We use $P(n,k)$ to denote the probability that a network of width
$n$ features and depth $k$ layers is alive, i.e.~not dead, the
estimation of which is the central study of this work. For convenience
in some proofs, we use $A_{k}\subset\mathbb{R}^{N(n,k)}$ denote the
event that the $k^{th}$ layer of the network is alive, a function
of the parameters $\theta$. Under this definition, $P(n,k)=P(A_{k})$.
\begin{figure}
\begin{centering}
\includegraphics[scale=0.5]{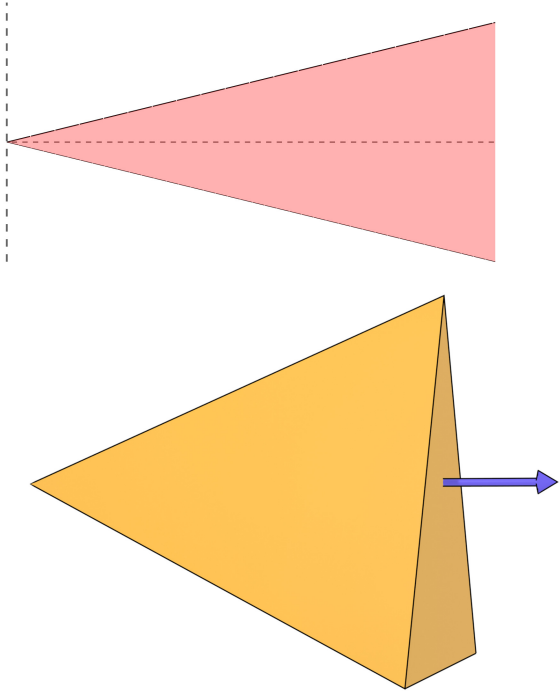}
\par\end{centering}
\caption{A convex cone, the typical dead set for a ReLU layer from $\mathbb{R}^{3}$
to itself. Top: projection in the $xy$-plane. Bottom: 3D rendering.
The blue arrow indicates extension of the cone to infinity.}
\label{fig:dead_sets}
\end{figure}

\section{Upper bound\label{sec:Upper-bound}}

First, we derive a conservative upper bound on $P(n,k)$. We know
that $S_{k}\subseteq\mathbb{R}_{+}^{n}$ for all $k>1$, due to the
ReLU nonlinearity. Thus, if layer $k>1$ kills $\mathbb{R}_{+}^{n}$,
then it also kills $S_{k}$ . Now, $R_{+}^{n}\subseteq D_{k}$ if
and only if none of the $n^{2}+n$ parameters are positive. Since
the parameters follow symmetric independent distributions, we have
$P(\mathbb{R}_{+}^{n}\subseteq D_{k})=2^{-n^{2}-n}$.

Next, note that layer $k$ is alive only if $k-1$ is, or more formally,
$A_{k}\subset A_{k-1}$. The law of total probability yields the recursive
relationship
\begin{equation}
P(n,k)=P(n,k|A_{k-1})P(n,k-1).\label{eq:layer_recursion}
\end{equation}
 Then, since $\mathbb{R}_{+}^{n}\subseteq D_{k}$ is independent of
$A_{k-1}$, we have the upper bound
\begin{align}
P(n,k) & \le\prod_{j=2}^{k}\left(1-P(\mathbb{R}_{+}^{n}\subseteq D_{k})\right)\nonumber \\
 & =\left(1-2^{-n^{2}-n}\right)^{k-1}.\label{eq:upper_bound}
\end{align}
Note that the bound can be sharpened to an exponent of $-n^{2}$ if
the bias terms are initialized to zero. For fixed $n$, this is a
geometric sequence in $k$. Note the limit is zero, verifying the
claim of Lu et al \citep{lu2019:dyingRelu}.

\section{Lower bound\label{sec:Lower-bound}}

The previous section showed that deep networks die in probability.
Luckily, we can bring our networks back to life again if they grow
deeper and wider simultaneously. Our insight is to show that, while
a deeper network has a higher chance of dead initialization, a wider
network has a lower chance. In what follows, we derive a lower bound
for $P(n,k)$ from which we compute the minimal width for a given
network depth. The basic idea is this: in a living network, $S_{k}\ne\emptyset$
for all $k$. Recall from equation \ref{eq:dead_set_neuron} that
the dead set of a neuron is a sub-level set of a continuous function,
thus it is a closed set \citep{Johnsonbaugh:1970:RealAnalysis}. Given
any point $x\in\mathbb{R}^{n}$, $x\notin D_{k}$ implies that there
exists some neighborhood around $x$ which is also not in $D_{k}$.
This means that a layer is alive so long as it contains a single living
point, so a lower bound for the probability of a living layer is the
probability of a single point being alive. 

It remains to compute the probability of a single point being alive,
the compliment of which we denote by $\gamma=\inf_{x\in\mathbb{R}^{n}}P(x\in D)$,
where $D$ is the dead set of some neuron. Surprisingly, $P(x\in D)$
does not depend on the value of $x$. Given some symmetric distribution
$\rho(a,b)$ of the parameters, we have
\begin{align}
P(x\in D) & =\int_{a\cdot x+b\le0}\rho(a,b)d(a,b).\label{eq:}
\end{align}
Now, the surface $a\cdot x+b=0$ can be rewritten as $(a\,b)\cdot(x,1)=0$,
a hyperplane in $\mathbb{R}^{2(n+1)}$. Since this surface has Lebesgue
measure zero, its contribution to the integral is negligible. Combining
this with the definition of a PDF, we have
\begin{align}
P(x\in D) & =1-\int_{a\cdot x+b\ge0}\rho(a,b)d(a,b)\label{eq:-1}
\end{align}
Then, change variables to $(\tilde{a},\tilde{b})=(-a,-b)$, with Jacobian
determinant $(-1)^{2(n+1)}=1$, and apply the symmetry of $f$ to
yield
\begin{align}
\int_{a\cdot x+b<0}\rho(a,b)d(a,b) & =1-\int_{\tilde{a}\cdot x+\tilde{b}<0}\rho(\tilde{a},\tilde{b})d(\tilde{a},\tilde{b}).\label{eq:-2}
\end{align}
Thus $\gamma=1-\gamma$, so $\gamma=1/2$. The previous calculation
can be understood in a simple way devoid of any formulas: for any
half-space containing $x$, the closure of its compliment is equally
likely and with probability one, exactly one of these sets contains
$x$. Thus, the probability of drawing a half-space containing $x$
is the same as the probability of drawing one that does not, so each
must have probability $1/2$.

Finally, if any of the $n$ neurons in layer $k$ is alive at $x$,
then the whole layer is alive. Since these are independent events
with probability $1/2$, $P(n,k|A_{k-1})\ge1-2^{-n}$ . It follows
from equation \ref{eq:layer_recursion} that
\begin{align}
P(n,k) & \ge(1-2^{-n})^{k}.\label{eq:lower_bound-1}
\end{align}

From this inequality, we can compute the width required to achieve
$P(n,k)=p$ as 
\begin{equation}
n=-\log_{2}(1-p^{1/k}).\label{eq:width_formula_with_gamma}
\end{equation}

See figure \ref{fig:lb_plots} for plots of these curves for various
values of $p$.

\begin{figure}
\begin{centering}
\includegraphics[scale=0.42]{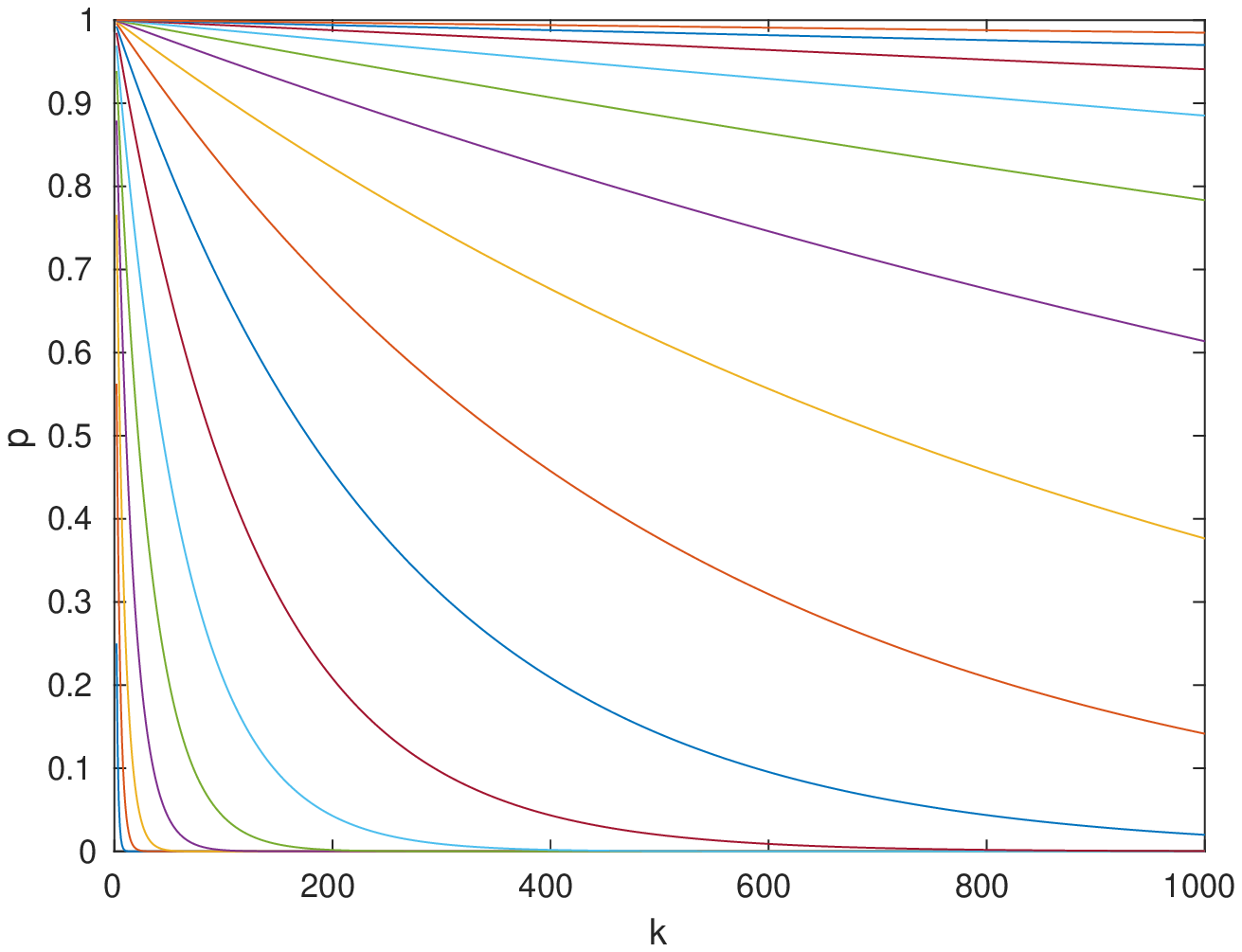}\includegraphics[scale=0.42]{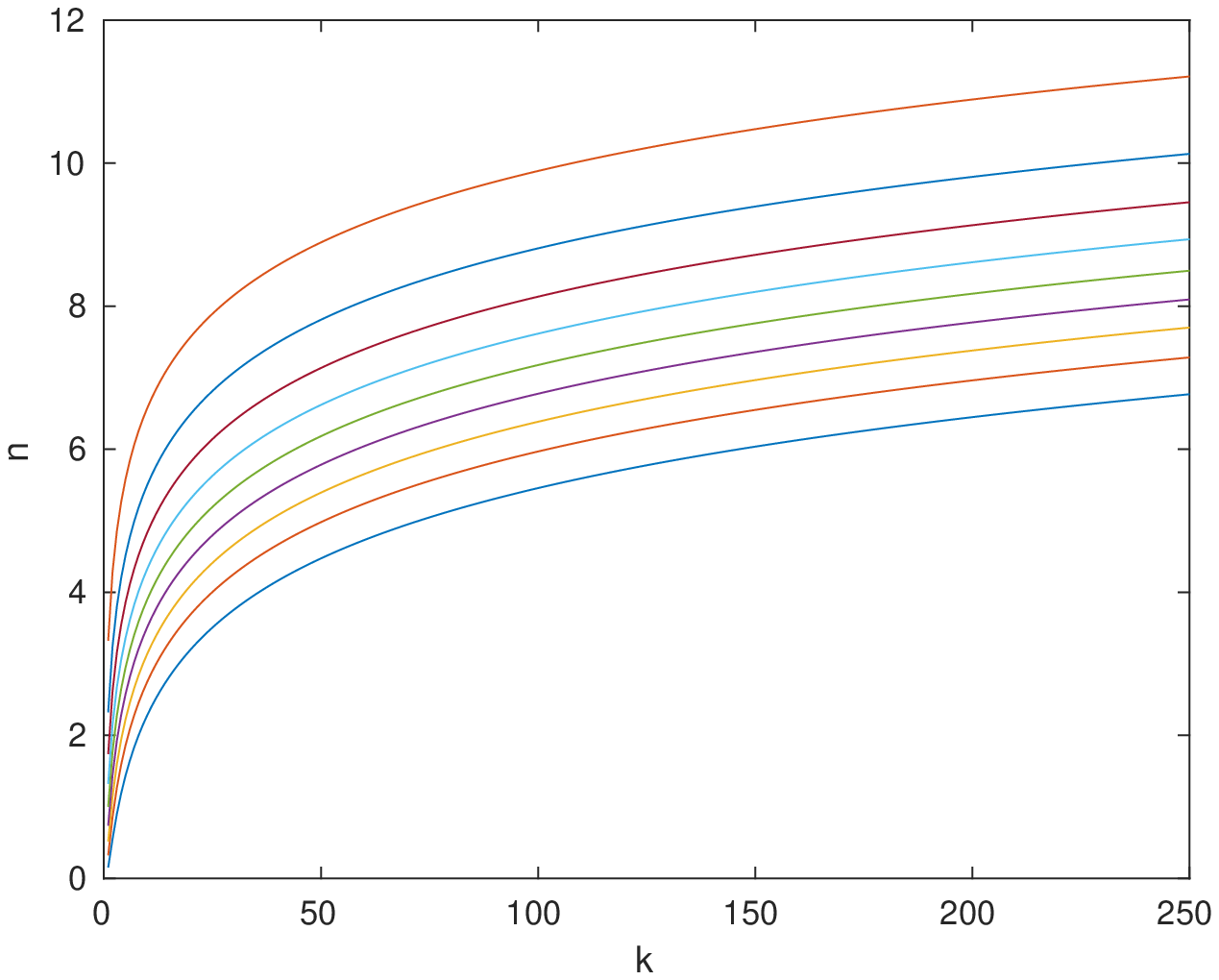}
\par\end{centering}
\caption{Left: Plot of $\ell(n,k)$ from equation \ref{eq:lower_bound-1},
with the number of features $n$ ranging in $[1,16]$. Right: Plot
of equation \ref{eq:width_formula_with_gamma} for $p$ ranging in
$[0.1,0.9]$. Shows the value of $n$ required to achieve $\ell(n,k)=p$
in a $k$-layer network. }
\label{fig:lb_plots}
\end{figure}

\section{Tightness of the bounds\label{sec:Tightness-of-the}}

Combining the results of the previous sections, we have the bounds
\begin{equation}
(1-2^{-n})^{k}\le P(n,k)\le(1-2^{-n^{2}-n})^{k-1}.\label{eq:both_bounds}
\end{equation}

Let $\ell(n,k)$ denote the lower bound, and $\upsilon(n,k)$ the
upper. We now show that these bounds are asymptotically tight, as
is later verified by our experiments. More precisely, $\upsilon(n,1)=P(n,1)$
for any $n$, with $S_{0}=\mathbb{R}^{n}$ . Furthermore, $P(n,k)\rightarrow\ell(n,k)$
along any path $n(k)$ such that $\lim_{k\rightarrow\infty}n(k)=\infty$.
This means that these bounds are optimal for the extreme cases of
a single-layer network, or a very deep network.

\subsection{Tightness of the upper bound for $k=1$}

First we show tightness of the upper bound, for a single-layer network.
Note that $\nu(n,1)=1$, so this is equivalent to showing that $P(n,1)=1$.
This is the case when the input set $S_{0}=\mathbb{R}^{n}$, and the
probability increases to this as the input set grows to fill $\mathbb{R}^{n}$.
This is formalized in the following theorem.

\begin{theorem}

Let $T_{1}\subseteq T_{2}\subseteq\dots$ be a sequence of nested
sets having $\mathbb{R}^{n}$ as their limit. Let the input set be
$S_{0}=T_{i}$, the $i^{th}$ set from this sequence. Then $P(n,1)\rightarrow1$
monotonically as $i\rightarrow\infty$.

\end{theorem}\label{theorem: prob_converge_1}

\begin{proof}

First note that $\mu(\mathbb{R}^{n})=1$. Next, let $\mu(T_{j})=P(n,k)$
with input set $T_{j}$. It is easy to verify that this satisfies
the axioms of a probability measure. Then we have 
\begin{align*}
\mathbb{R}^{n} & =\cup_{i=1}^{\infty}T_{i}\\
\mu(\mathbb{R}^{n}) & =\lim_{i\rightarrow\infty}\mu\left(T_{i}\right),
\end{align*}

whence converge and monotonicity follow from the axioms of measure
theory. 

\end{proof}

\subsection{Tightness of the lower bound with contracting input set\label{subsec:Tightness-of-lb-contracting-input}}

Next we show tightness of the lower bound. Recalling that our upper
bound is tight for a single-layer network with all of $\mathbb{R}^{n}$
as its input, it stands to reason that the lower bound will be tight
for either smaller input sets or more complex models. This intuition
is correct, as we will show rigorously in lemma \ref{lemma:prob_x_alive}.
In preparation, notice from the derivation of $\ell(n,k)$ in section
\ref{sec:Lower-bound} that it equals the probability that some data
point $x\in\mathbb{R}^{n}$ will be alive in a random network. Owing
to the symmetry of the parameter distribution, this does not even
depend on the value of $x$, except in the technical case that $x=0$
and the biases are all zero. This furnishes the following theorem,
which is verified experimentally by figure \ref{fig:grid_results}
in section \ref{sec:Experimental-validation}.

\begin{lemma}\label{lemma:prob_x_alive}

Let $x\in\mathbb{R}^{n}\backslash\{0\}$. Then the probability that
$x$ is alive in a $k$-layer network is $\ell(n,k)$. 

\end{lemma}

\begin{proof}

The argument is essentially the same as that of section \ref{sec:Lower-bound}.
By symmetry, the probability that a $x$ is dies in a single neuron
is $1/2$. By independence, the probability that $x$ dies in a single
layer is $1-2^{-n}$. Again by conditional independence, this extends
to $k$ layers as $P(x\in D)=(1-2^{-n})^{k}$.

\end{proof}

Equipped with this result, it is easy to see that $P(n,k)\rightarrow\ell(n,k)$
as the input set shrinks in size to a single point. This is made rigorous
in the following theorem.

\begin{theorem}

Let $T_{1}\supseteq T_{2}\supseteq\dots$ be a sequence of nested
sets with limit $x\in\mathbb{R}^{n}\backslash\{0\}$. Let the input
set be $S_{0}=T_{i}$, the $i^{th}$ set from this sequence. Then
$P(n,k)\rightarrow\ell(n,k)$ monotonically as $i\rightarrow\infty$.

\end{theorem}\label{theorem:converge_to_lb_r_shrinks}

\begin{proof}

The proof is the same as that of theorem \ref{theorem: prob_converge_1},
except we take a decreasing sequence. Then, using the same notation
as in that proof, $\mu(S_{0}(x,r_{i}))\rightarrow\mu(x)$. Finally,
lemma \ref{lemma:prob_x_alive} tells us that $\mu(x)=\ell(n,k)$. 

\end{proof}

We have shown that $P(n,k)=\ell(n,k)$ when the input set is a single
point. Informally, it is easy to believe that $P(n,k)\rightarrow\ell(n,k)$
as the network grows deeper, for any input set. The previous theorem
handles the case that the input set shrinks, and we can often view
a deeper layer as one with a shrunken input set, as the network tends
to constrict the inputs to a smaller range of values. To formalize
this notion, define the \textbf{living set }for parameters $\theta$
as $L_{k}(\theta)=S_{0}\cap F_{1}^{-1}(D_{1})^{c}\cap\dots\cap F_{k}^{-1}(D_{k})^{c}$.
Then we can state the following lemma and theorem.

\begin{lemma}

The expected measure of the living set $L_{k}$ over all possible
network parameters $\theta$ is equal to $\ell(n,k).$

\end{lemma}

\begin{proof}

Write the expected measure as $\mathbb{E_{\theta}}\int_{x\in S_{0}}f(x)P(x\in L_{k})$.
Then apply Fubini's theorem followed by lemma \ref{lemma:prob_x_alive}.

\end{proof}

\begin{theorem}

Let $n(k)$ be a path through hyperparameter space such that $n$
is a non-decreasing function of $k$, and $\ell(n(k),k)\rightarrow0$.
Then $P(n(k),k)\rightarrow0$ as well.

\end{theorem}

\begin{proof}

For increasing $k$ with fixed $n$, note that $L_{k}$ is a decreasing
nested sequence of sets. To move beyond fixed $n$, consider the infinite
product space generated by all possible values of $n$, $k$ and $\theta$.
In this space, $L_{k}$ is also a nested decreasing sequence, as $n$
is non-decreasing with $k$. By the lemma, $\ell(n(k),k)\rightarrow0$
implies the limiting set $\cap_{k=1}^{\infty}L_{k}$ has measure zero.
Then, we know that $\mu(L_{k})\rightarrow0$ monotonically as well. 

Next, recall from section \ref{sec:Lower-bound} that dead sets are
closed, and since the network is continuous we know that $L_{k}$
is either empty or open, as it is a finite intersection of such sets.
Since the input data are drawn from a continuous distribution, it
follows that $\mu(L_{k})>0$ if and only if $L_{k}\ne\emptyset$.
Then, we can write
\[
P(n,k)=\mathbb{E_{\theta}}\begin{cases}
1, & \mu(L_{k}(\theta))>0\\
0, & \mbox{otherwise.}
\end{cases}
\]
From this formulation, the integrand is monotone in $k$, so $P(n(k),k)\rightarrow0$
by the monotone convergence theorem. 

\end{proof}

See figure \ref{fig:lb_plots} for example paths $n(k)$ satisfying
the conditions of this theorem. In general, our bounds imply that
deeper networks must be wider, so all $n(k)$ of practical interest
ought to be non-decreasing. This concludes our results from measure
theory. In the next section, we adopt a statistical approach to remove
the requirement that $\ell(n,k)\rightarrow0$.

\subsection{Tightness of the lower bound as $n(k)\rightarrow\infty$}

We have previously shown that $P(n,k)$ converges to the upper bound
$\nu(n,k)$ as the input set grows to $\mathbb{R}^{n}$ with $k=1$,
and that it converges to the lower bound $\ell(n,k)$ as the input
set shrinks to a single point, for any $k$. Finally, we argue that
$p(n,k)\rightarrow\ell(n,k)$ as $n$ and $k$ increase without limit,
without reference to the input set. We have already shown that this
holds by an expedient argument in the special case that $\ell(n,k)\rightarrow0$.
In what follows, we drop this assumption and try to establish the
result for the general case, substituting an assumption on $\ell(n,k)$
for one on the statistics of the network outputs.

\begin{assumption}{}\label{assumption:variance}Let $\sigma(F_{k}(x)|\theta)$
denote the vector of standard deviations of each element of $F_{k}(x)$,
which varies with respect to the input data $x$, conditioned on parameters
$\theta$. Let $\lambda_{k}=\mathbb{E}[F_{k}(x)|\theta]$. Assume
the sum of normalized conditional variances, taken over all living
layers as
\begin{equation}
\Sigma(n)=\sum_{k=1}^{\infty}\mathbb{E}_{\theta|A_{k}}\frac{\|\sigma(F_{k}(x)|\theta)\|^{2}}{\|\lambda_{k}\|^{2}}\label{eq:-3}
\end{equation}
has $\Sigma(n)\in o(2^{n})$. \end{assumption}{}

This basically states that the normalized variance decays rapidly
with the number of layers $k$, and it does not grow too quickly with
the number of features $n$. The assumption holds in the numerical
simulations of section \ref{sec:Experimental-validation}. In what
follows, we argue that this is a reasonable assumption, based on some
facts about the output variance of a ReLU network.

First we compute the output variance of the affine portion of a ReLU
neuron. Assume the input features have zero mean and identity covariance
matrix, i.e.~the data have undergone a whitening transformation.
Let $\sigma_{\theta}^{2}$ denote the variance of each network parameter,
and let $F_{1,j}(x)$ denote the output of the $j^{th}$ neuron in
the first layer, with $\tilde{F}_{1,j}(x)$, $\lambda_{j}$ and parameters
$(a_{j},b_{j})$ defined accordingly. Then 
\begin{align}
\mathbb{E}_{\theta}[\sigma^{2}(\tilde{F}_{1,j}|\theta)] & =\sum_{i=1}^{n}\mathbb{E}[a_{j,i}^{2}]\sigma^{2}(x_{i})\label{eq:-4}\\
 & =n\sigma_{\theta}^{2}.\nonumber 
\end{align}
This shows that the pre-ReLU variance does not increase with each
layer so long as $\sigma_{\theta}^{2}\le1/n$, the variance used in
the Xavier initialization scheme \citep{Glorot:2010:XavierInit}.
A similar computation can be applied to later layers, canceling terms
involving uncorrelated features. 

To show that the variance actually decreases rapidly with $k$, we
must factor in the ReLU nonlinearity. The basic intuition is that,
as data progresses through a network, information is lost when neurons
die. For a deep enough network, the response of each data point is
essentially just $\lambda_{k}$, which we call the eigenvalue of the
network. The following result makes precise this notion of information
loss.

\begin{lemma}{}\label{lem:relu_variance}Given a rectifier neuron
$F$ with input data $x$ and parameters $\theta$, the expected output
variance is related to the pre-ReLU variance by

\begin{align*}
\mathbb{E}_{\theta}\sigma^{2}(F(x)|\theta) & =\frac{1}{2}\mathbb{E}_{\theta}\sigma^{2}(\tilde{F}(x)|\theta)-\mathbb{E}_{\theta}\lambda^{2}.
\end{align*}

\end{lemma}{}

\begin{proof}

From the definition of variance,
\begin{equation}
\sigma^{2}(F(x)|\theta)=\mathbb{E}_{x}[F^{2}(x)|\theta]-\lambda^{2}.\label{eq:cond_variance}
\end{equation}
Taking the expected value over $\theta$, we have that $\mathbb{E}_{\theta}\mathbb{E}_{x}[F^{2}(x)|\theta]=\mathbb{E}_{x}\mathbb{E}_{\theta}[F^{2}(\theta)|x]$,
both being equal to $\mathbb{E}_{\theta,x}[F^{2}(x,\theta)]$. For
a symmetric zero-mean parameter distribution $\rho(\theta)$ we have
\[
\mathbb{E}_{\theta}\left[F^{2}(\theta|x)\right]=\int_{\tilde{F}(x,\theta)>0}\tilde{F}^{2}(x,\theta)\rho(\theta)d\theta.
\]
As in section \ref{sec:Lower-bound}, symmetry tells us that the integral
over $\tilde{F}>0$ is equal to that over $\tilde{F}\le0$, thus $\mathbb{E}_{\theta}\left[F^{2}(\theta|x)\right]=\frac{1}{2}\mathbb{E}_{\theta}\left[\tilde{F}^{2}(\theta|x)\right].$
Similarly, odd symmetry of $\tilde{F}$ yields $\mathbb{E}_{\theta}\tilde{\lambda}=0$,
so the pre-ReLU variance is just $\mathbb{E}_{\theta}(\tilde{F}^{2}(x|\theta))$.
Combining these results with equation \ref{eq:cond_variance} yields
the desired result.

\end{proof}

This shows that the average variance decays in each layer by a factor
of at least $1/2$, explaining the factor of $2$ applied to parameter
variances in the He initialization scheme \citep{He:2015:HeInitialization}.
However, the remaining $\lambda$ term shows that a factor of 2 is
insufficient, and variance decay will still occur, as seen in the
experiments of section \ref{sec:Experimental-validation}. This serves
both as a principled derivation of He initialization, as well as an
elucidation of its weaknesses. While this is not a complete proof
of assumption \ref{assumption:variance}, it offers strong evidence
for geometric variance decay.

After these preliminaries, we can now show the tightness of the lower
bound based on assumption \ref{assumption:variance}. To simplify
the proof, we switch to the perspective of a discrete IID dataset
$x_{1},\dots,x_{M}\in S_{0}$. Let $\mathcal{A}(k)$ denote the number
of living data points at layer $k$, and let $\mathcal{P}(n,k)=P(\mathcal{A}(k)>0)$.
It is a vacuous to study $k\rightarrow\infty$ with fixed $n$, as
$\nu,\ell\rightarrow0$ in this case. Similarly, $\nu,\ell\rightarrow1$
as $n\rightarrow\infty$ for fixed $k$. Instead, we consider any
path $n(k)\rightarrow\infty$ through hyperparameter space, which
essentially states that both $n$ and $k$ grow without limit. For
example, consider equation \ref{eq:width_formula_with_gamma} for
any $p\in(0,1)$, which gives the curve $n(k)$ such that $\ell(n(k),k)=p$.
Then, we show that $\mathcal{P}(n,k)\rightarrow\ell(n,k)$ as $n\rightarrow\infty$. 

\begin{proof}

Partition the parameter space into the following events, conditioned
on $\mathcal{A}_{k-1}>0$:
\begin{enumerate}
\item $E_{1}$: All remaining data lives at layer $k$: $\mathcal{A}(k)=\mathcal{A}(k-1)$
\item $E_{2}$: Only some of the remaining data lives at layer $k$: $0<\mathcal{A}(k)<\mathcal{A}(k-1)$
\item $E_{3}$: All the remaining data is killed: $0=\mathcal{A}(k)<\mathcal{A}(k-1)$.
\end{enumerate}
\begin{figure}
\includegraphics[scale=0.39]{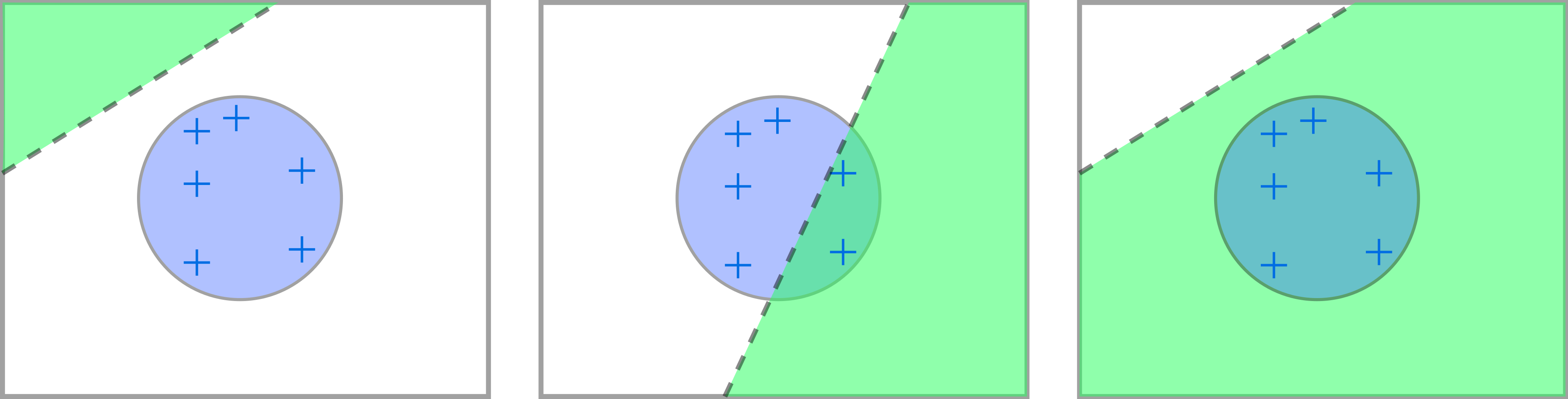}

\caption{Visualization of possible neuron death events, $E_{1}$, $E_{2}$
and $E_{3}$ from left to right. The blue circle represents the input
data distribution $S_{0}$, the crosses represent the actual discrete
data $x_{1},\dots,x_{M}$, and the green shaded region represents
the dead set. Note that for a single neuron, $E_{1}$ and $E_{3}$
are geometric complements, each having equal probability.}
\label{fig:events}
\end{figure}
These events are visualized in figure \ref{fig:events}. Now, as the
events are disjoint, we can write
\begin{align}
P(E_{1}) & =P(E_{1}|E_{2}^{C})\left(1-P(E_{2})\right).\label{eq:-9}
\end{align}

Now, conditioning on the compliment of $E_{2}$ means that either
all the remaining data is alive or none of it. If none of it, then
flipping the sign of any neuron brings all the remaining data to life
with probability one. As in section \ref{sec:Lower-bound}, symmetry
of the parameter distribution ensures all sign flips are equally likely.
Thus $P(E_{1}|E_{2}^{C})=1-2^{-n}$. Since $E_{1}$ and $E_{2}$ partition
$\{\mathcal{A}_{k}>0\}$, we have
\begin{align}
\mathcal{P}(n,k) & =\mathcal{P}(n,k-1)\left(P(E_{1})+P(E_{2})\right)\label{eq:-10}\\
 & =\mathcal{P}(n,k-1)\left((1-2^{-n})+2^{-n}P(E_{2})\right).\nonumber 
\end{align}
Recall the lower bound $\ell(n,k)=(1-2^{-n})^{k}.$ Expanding terms
in the recursion, and letting $E_{2}(j)$ be the analogous event for
layer $j$, we have that
\begin{align}
\frac{\mathcal{P}(n,k)}{\ell(n,k)} & =1+\frac{1}{2^{n}-1}\sum_{j=1}^{k}P(E_{2}(j))+\dots\label{eq:prob_over_lb_ellipsis}
\end{align}
Now, $E_{2}$ implies that $\|F_{j}(x_{m})\|=0$ for some $m\in\{1,\dots,M\}$
while $\|\lambda_{j}\|>0$. But, $F_{j}(x)=0$ if and only if $\|F_{j}(x)-\lambda_{j}\|\ge\lambda_{j}$.
Applying the union bound across data points followed by Chebyshev's
inequality, we have that
\begin{align}
P(E_{2}(j)) & \le M\mathbb{E}_{\theta|A_{j}}\frac{\|\sigma(F_{j}(x)|\theta)\|^{2}}{\|\lambda_{j}\|^{2}}.\label{eq:-11}
\end{align}
Therefore $\sum_{j=1}^{k}P(E_{2}(j))\le M\Sigma(n)$, the sum from
assumption \ref{assumption:variance}. Now, since all terms in equation
\ref{eq:prob_over_lb_ellipsis} are non-negative, we can collect them
in the bound
\begin{align}
\frac{\mathcal{P}(n,k)}{\ell(n,k)} & \le\sum_{h=0}^{\infty}\left(\frac{M\Sigma(n)}{2^{n}-1}\right)^{h}.\label{eq:-13}
\end{align}
This is a geometric series in $r(n)=M\Sigma(n)/(2^{n}-1)$, and by
assumption \ref{assumption:variance}, $r(n)\rightarrow0$ with $n$.

\end{proof}

This concludes our argument for the optimality of our lower bound.
Unlike in the previous section, here we required a strong statistical
assumption to yield a compact, tidy proof. The truth of assumption
\ref{assumption:variance} depends on the distribution of both the
input data and the parameters. Note that this is a sufficient rather
than necessary condition, i.e.~the lower bound might be optimal regardless
of the assumption. Note also that assumption \ref{assumption:variance}
is a proxy for the more technical condition that the sum $P(E_{2}(1))+P(E_{2}(2))+\dots$
is $o(2^{n})$. This means that the probability of the dataset partially
dying decreases rapidly with network depth. In other words, for deeper
layers, either all the data lives or none of it. There may be other
justifications for this result which dispense with the variance, such
as the contracting radius of the output set, or the increasing dimensionality
of the feature space $n$. We chose to work with the variance because
it roughly captures the size of the output set, and is much easier
to compute than the geometric radius or spectral norm.

\begin{figure}[H]
\begin{centering}
\includegraphics[scale=0.62]{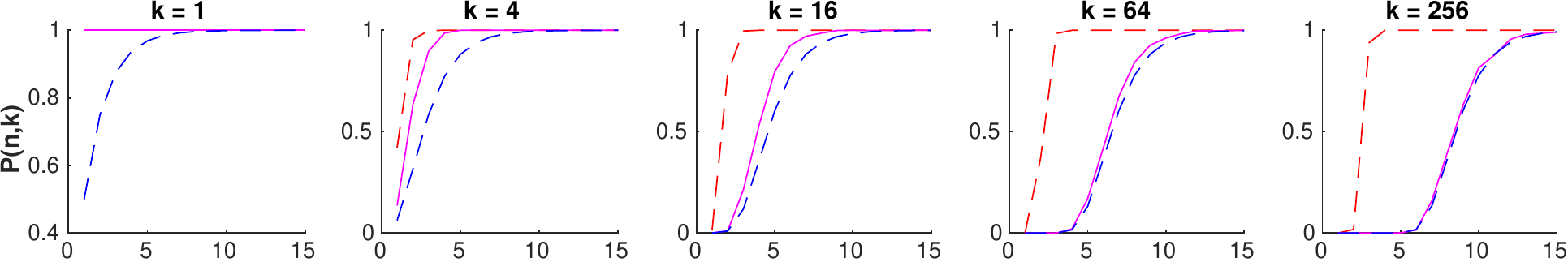}
\par\end{centering}
\caption{Probability of living network $P(n,k)$ (solid) versus bounds $\nu(n,k)$
and $\ell(n,k)$ (dashed).}
\label{fig:grid_results}
\end{figure}

\begin{figure}[t]
\begin{centering}
\includegraphics[scale=0.61]{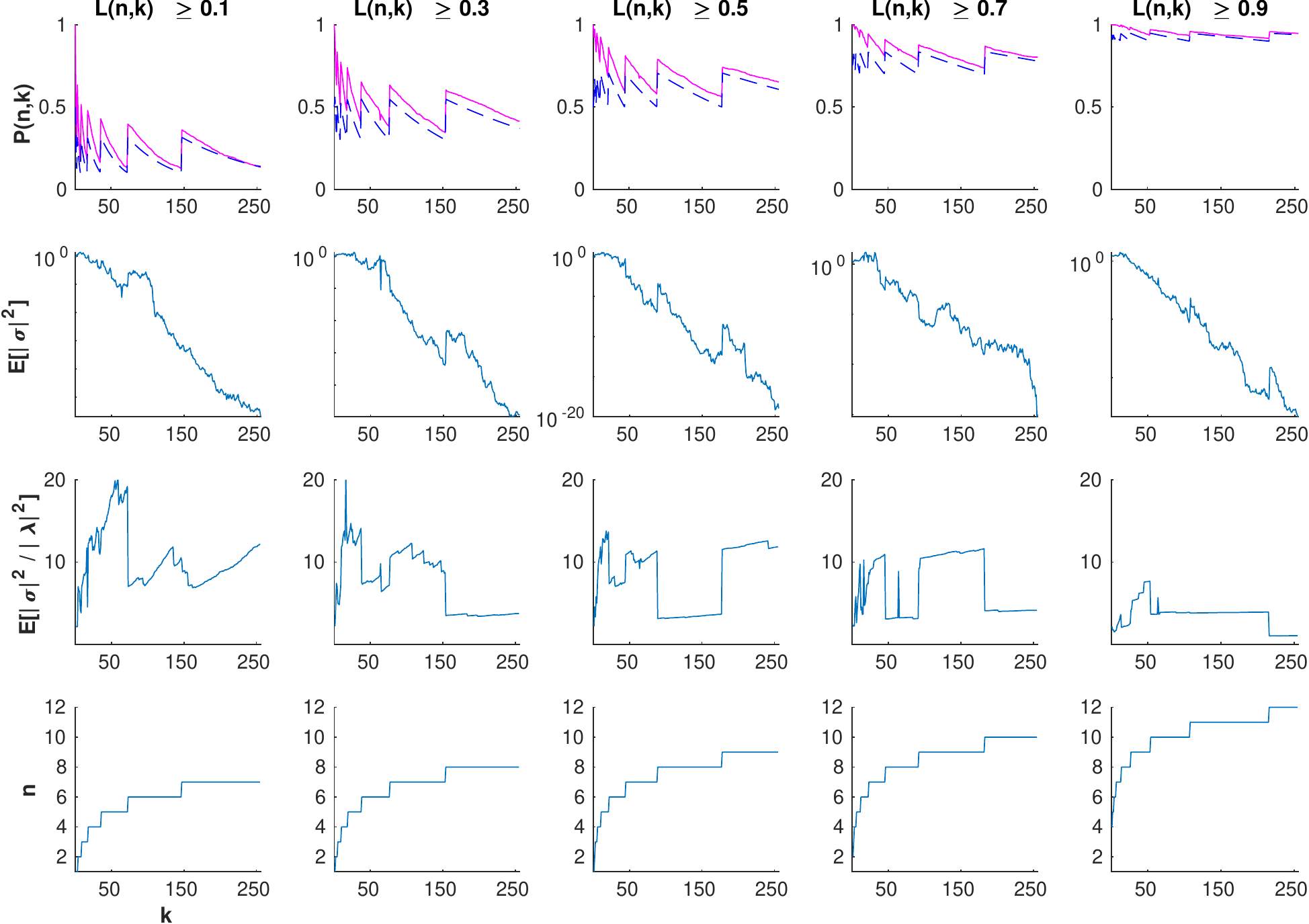}
\par\end{centering}
\caption{Top\textendash probability of living network $P(n,k)$ (solid) versus
lower bound $\ell(n,k)$ (dashed). Middle\textendash normalized variance
from assumption \ref{assumption:variance}, on a log scale. Bottom\textendash $n(k)$
derived from rounding equation \ref{eq:width_formula_with_gamma},
providing a roughly constant value of $\ell(n,k)$.}
\label{fig:lb_constant_versus_p}
\end{figure}

\section{Numerical simulations of bounds\label{sec:Experimental-validation}}

We now confirm our previously established bounds on neuron death by
numerical simulation. We followed the straightforward methodology
of generating a number of random networks, applying these to random
data, and computing various statistics on the output. All experiments
were conducted using SciPy and required several hours' compute time
on a standard desktop CPU \citep{Virtanen:2020:SciPy}.

First we generated the random data and networks. Following the typical
convention of using white training data, we randomly generated $1024$
data points from standard normal distributions in $\mathbb{R}^{n}$,
for each feature dimensionality $n\in\{1,\dots,15\}$ \citep{Sola:1997:dataNormalization}.
Then, we randomly initialized $1024$ fully-connected ReLU networks
having $1\le k\le256$ layers. Following the popular He initialization
scheme, we initialized the multiplicative network parameters with
normal distributions having variance $\sigma_{\theta}^{2}=2/n$, while
the additive parameters were initialized to zero \citep{He:2015:HeInitialization}.
Then, we computed the number of living data points for each network,
concluding that a network has died if no living data remain at the
output. The results are shown in figure \ref{fig:grid_results} alongside
the upper and lower bounds from inequality \ref{eq:both_bounds}.

Next, to show asymptotic tightness of the lower bound, we conducted
a similar experiment using the hyperparameters defined by equation
\ref{eq:width_formula_with_gamma}, rounding $n(k)$ upwards to the
next integer, for various values of $p$. These plots ensure that
the lower bound is approximately constant at $p$, while the upper
bound is approximately equal to one, that $P(n,k)$ converges to a
nontrivial value. The results are shown in figure \ref{fig:lb_constant_versus_p}.

We draw two main conclusions from these graphs. First, the simulations
agree with our theoretical bounds, as $\ell(n,k)\le P(n,k)\le\nu(n,k)$
in all cases. Second, the experiments support asymptotic tightness
of the lower bound, as $P(n,k)\rightarrow\ell(n,k)$ with increasing
$n$ and $k$. Interestingly, the variance decreases exponentially
in figure \ref{fig:lb_constant_versus_p}, despite the fact that He
initialization was designed to stabilize it \citep{He:2015:HeInitialization}.
This could be explained by the extra term in lemma \ref{lemma:prob_x_alive}.
However, we do not notice exponential decay after normalization by
the mean $|\lambda|^{2}$, as would be required by assumption \ref{assumption:variance}.
Recall that assumption \ref{assumption:variance} is a strong claim
meant to furnish a sufficient, yet unnecessary condition that $P(n,k)\rightarrow\nu(n,k)$.
That is, the lower bound might still be optimal even if the assumption
is violated, as suggested by the graphs.

\section{Sign-flipping and batch normalization\label{sec:Living-initialization-scheme}}

Our previous experiments and theoretical results demonstrated some
of the issues with IID ReLU networks. Now, we propose a slight deviation
from the usual initialization strategy which partially circumvents
the issue of dead neurons. As in the experiments, we switch to a discrete
point of view, using a finite dataset $x_{1},\dots,x_{M}\in\mathbb{R}^{n}$.
Our key observation is that, with probability one, negating the parameters
of a layer revives any data points killed by it. Thus, starting from
the first layer, we can count how many data points are killed, and
if this exceeds half the number which were previously alive, we negate
the parameters of that layer. This alters the joint parameter distribution
based on the available data, while preserving the marginal distribution
of each parameter. Furthermore, this scheme is practical, costing
essentially the same as a forward pass over the dataset.

For a $k$-layer network, this scheme guarantees that at least $\left\lfloor M/2^{k}\right\rfloor $
data points live. The caveat is that the pre-ReLU output cannot be
all zeros, or else negating the parameters would not change the output.
However, this edge case has probability zero since $\tilde{F}_{k}$
depends on a polynomial in the parameters of layer $k$, the roots
of which have measure zero. Batch normalization provides a similar
guarantee: if $\mathbb{E}[\tilde{F}_{k}(x)|\theta]=0$ for all $k$,
then with probability one there is at least one living data point
\citep{Ioffe:2015:batchNorm}. Similar to our sign-flipping scheme,
this prevents total network death while still permitting individual
data points to die. 

These two schemes are simulated and compared in figure \ref{fig:living_data_results}.
Note that for IID data $P(\alpha)=\ell(n,k)$, as that bound comes
from the probability that any random data point is alive. Sign-flipping
significantly increased the proportion of living data, far exceeding
the $2^{-k}$ fraction that is theoretically guaranteed. In contrast,
batch normalization ensured that the network was alive, but did nothing
to increase the proportion of living data over our IID baseline.

\begin{figure}[H]
\begin{centering}
\includegraphics[scale=0.61]{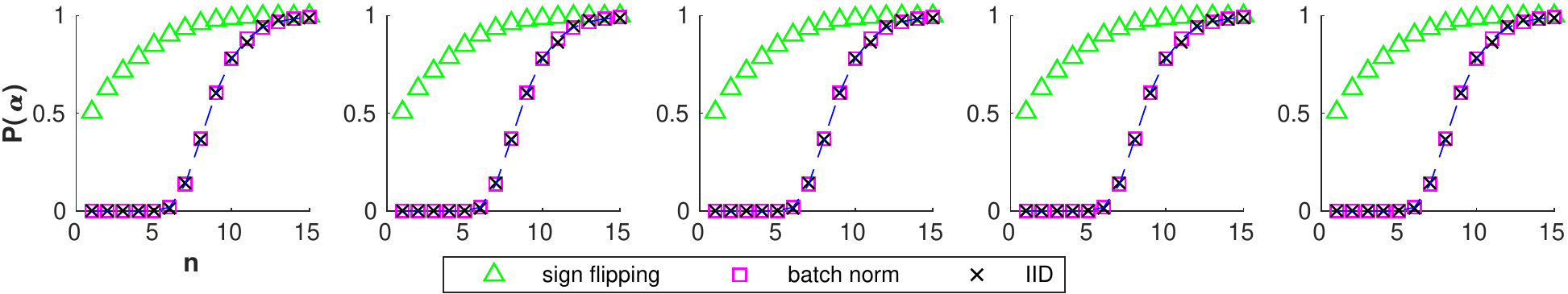}
\par\end{centering}
\caption{Probability of living data $P(\alpha)$ for various parameter distributions,
with lower bound $\ell(n,k)$ (dashed).}
\label{fig:living_data_results}
\end{figure}

\section{Implications for network design}

Neuron death is not only relevant for understanding how to initialize
a very deep neural network, but also helps to explain various aspects
of the model architecture itself. Equation \ref{eq:width_formula_with_gamma}
in section \ref{sec:Lower-bound} tells us how wide a $k$-layer network
needs to be; if a network is deep and skinny, most of the data will
die. We have also seen in section \ref{sec:Living-initialization-scheme}
that batch normalization prevents neuron death, supplementing the
existing explanations of its efficacy \citep{Santurkar:2018:batchNormHelpOptimization}.
Surprisingly, many other innovations in network design can also be
viewed through the lens of neuron death. 

Up to this point we focused on fully-connected ReLU networks. Here
we briefly discuss how our results generalize to other network types.
Interestingly, many of the existing network architectures bear relevance
in preventing neuron death, even if they were not originally designed
for this purpose. What follows is meant to highlight the major categories
of feed-forward models, but our list is by no means exhaustive.

\subsection{Convolutional networks}

Convolutional neural networks (CNNs) are perhaps the most popular
type of artificial neural network. Specialized for image and signal
processing, these networks have most of their neurons constrained
to perform convolutions. Keeping with our previous conventions, we
say that a convolutional neuron takes in feature maps $X_{1},\dots,X_{\mathcal{N}}\in\mathbb{R}^{d\times d}$
and computes 
\begin{equation}
f(X_{1},\dots,X_{\mathcal{N}})=\max\left\{ b+\sum_{l=1}^{\mathcal{N}}A_{l}\ast X_{l},0\right\} ,\label{eq:-14}
\end{equation}
where the maximum is again taken element-wise. In the two-dimensional
case, $b\in\mathbb{R}$ and $A_{l}\in\mathbb{R}^{\mathcal{M}\times\mathcal{M}}$.
By the Riesz Representation Theorem, discrete convolution can be represented
by multiplication with some matrix $\tilde{A}_{l}\in\mathbb{R}^{d\times d}$.
Since $\tilde{A}_{l}$ is a function of the much smaller matrix $A_{l}$,
we need to rework our previous bounds in terms of the dimensions $\mathcal{N}$
and $\mathcal{M}$. It can be shown by similar arguments that
\begin{equation}
(1-2^{-\mathcal{N}})^{k}\le P(d,\mathcal{N},\mathcal{M},k)\le\left(1-2^{-\mathcal{N}(\mathcal{M}^{2}+1)}\right)^{k-1}.\label{eq:-15}
\end{equation}
Compare this to inequality \ref{eq:both_bounds}. As with fully-connected
networks, the lower bound depends on the number of neurons, while
the upper bound depends on the total number of parameters. To our
knowledge, the main results of this work apply equally well to convolutional
networks; the fully-connected type was used only for convenience.

\subsection{Residual networks and skip connections}

Residual networks (ResNets) are composed of layers which add their
input features to the output \citep{He:2016:ResNet}. Residual connections
do not prevent neuron death, as attested by other recent work \citep{arnekvist:2020:dyingReluMomentum}.
However, they do prevent a dead layer from automatically killing any
later ones, by creating a path around it. This could explain how residual
connections allow deeper networks to be trained \citep{He:2016:ResNet}.
The residual connection may also affect the output variance and prevent
information loss.

A related design feature is the use of skip connections, as in fully-convolutional
networks \citep{Shelhamer:2017:FCN}. For these architectures, the
probability of network death is a function of the shortest path through
the network. In many segmentation models this path is extremely short.
For example, in the popular U-Net architecture, with chunk length
$L$, the shortest path has a depth of only $2L$ \citep{Ronneberger:2015:Unet}.
This means that the network can continue to function even if the innermost
layers are disabled.

\subsection{Other nonlinearities}

This work focuses on the properties of the ReLU nonlinearity, but
others are sometimes used including the leaky ReLU, ``swish'' and
hyperbolic tangent \citep{Djork:2016:ELU,Ramachandran:2018:SwishNonlinearity}.
With the exception of the sigmoid type, all these alternatives retain
the basic shape of the ReLU, with only slight modifications to increase
smoothness or prevent neuron death. For all of these functions, part
of the input domain is much less sensitive than the rest. In the sigmoid
type it is the extremes of the input space, while in the ReLU variants
it the negatives. Our theory easily extends to all the ReLU variants
by replacing dead data points with \textit{weak} ones, i.e. those
with small gradients. Given that gradients are averaged over a mini-batch,
weak data points are likely equivalent to dead ones in practice \citep{Goodfellow:2016:DLBook}.
We suspect this practical equivalence is what allows the ReLU variants
to maintain similar performance levels to the original \citep{Ramachandran:2018:SwishNonlinearity}.

\section{Conclusions}

We have established several theoretical results concerning ReLU neuron
death, confirmed these results experimentally, and discussed their
implications for network design. The main results were an upper and
lower bound for the probability of network death $P(n,k)$, in terms
of the network width $n$ and depth $k$. We provided several arguments
for the asymptotic tightness of these bounds. From the lower bound,
we also derived a formula relating network width to depth in order
to guarantee a desired probability of living initialization. Finally,
we showed that our lower bound coincides with the probability of data
death, and developed a sign-flipping initialization scheme to reduce
this probability. 

We have seen that neuron death is a deep topic covering many aspects
of neural network design. This article only discusses what happens
at random initialization, and much more could be said about the complex
dynamics of training a neural network \citep{arnekvist:2020:dyingReluMomentum,wu2018:AdaptiveNetworkScaling}.
We do not claim that neuron death offers a complete explanation of
any of these topics. On the contrary, it adds to the list of possible
explanations, one of which may someday rise to preeminence. Even if
neuron death proves unimportant, it may still correlate to some other
quantity with a more direct relationship to model efficacy. In this
way, studying neuron death could lead to the discovery of other important
factors, such as information loss through a deep network. At any rate,
the theory of deep learning undoubtedly lags behind the recent advances
in practice, and until this trend reverses, every avenue is deserving
of exploration.

\bibliographystyle{plain}
\bibliography{../../neurips2020}

\end{document}